\documentclass[12pt]{alt2022}

\title[Learning with Distributional Inverters]{Learning with distributional inverters}
\usepackage{times}
\usepackage{alt-addendum}

\newcommand{\commentstyle}[1]{#1}

\def\ShowAuthNotes{1}
\ifnum\ShowAuthNotes=1
\newcommand{\mcnote}[1]{[{\footnotesize  \textcolor{red}{\bf Marco:} \commentstyle{#1}}]}
\newcommand{\msnote}[1]{[{\footnotesize  \textcolor{blue}{\bf Manuel:} \commentstyle{#1}}]}
\newcommand{\ebnote}[1]{[{\footnotesize  \textcolor{purple}{\bf Eric:} \commentstyle{#1}}]}
\newcommand{\aknote}[1]{[{\footnotesize  \textcolor{green}{\bf Antonina:} \commentstyle{#1}}]}
\newcommand{\rnote}[1]{[{\footnotesize  \textcolor{orange}{\bf Ramyaa:} \commentstyle{#1}}]}
\else
\newcommand{\mcnote}[1]{}
\newcommand{\msnote}[1]{}
\newcommand{\ebnote}[1]{}
\newcommand{\aknote}[1]{}
\newcommand{\rnote}[1]{}
\fi

\altauthor{\Name{Eric Binnendyk} \Email{eric.binnendyk@student.nmt.edu}\\
 \addr
 \AND
 \Name{Marco Carmosino}
 \Email{mcarmosi@bu.edu}\\
 \addr \AND
 \Name{Antonina Kolokolova}
 \Email{kol@cs.toronto.edu}\\
 \addr \AND
 \Name{Ramyaa Ramyaa}
 \Email{ramyaa.ramyaa@nmt.edu}\\
 \addr \AND
 \Name{Manuel Sabin}
 \Email{msabin27@gmail.com}\\
 \addr }

\begin{document}

\maketitle

\begin{abstract}We generalize the ``indirect learning'' technique of
\citet{DBLP:conf/colt/FurstJS91} to reduce from learning a concept
class over a samplable distribution $\mu$ to learning the same concept
class over the uniform distribution. The reduction succeeds when the
sampler for $\mu$ is both contained in the target concept class and
efficiently invertible in the sense of
\citet{DBLP:conf/focs/ImpagliazzoL89}. We give two applications.

\begin{itemize}
\item We show that $\AC^0[q]$ is learnable over any succinctly-described product
  distribution. $\AC^0[q]$ is the class of constant-depth Boolean
  circuits of polynomial size with AND, OR, NOT, and counting modulo
  $q$ gates of unbounded fanins. Our algorithm runs in randomized
  quasi-polynomial time and uses membership queries.
\item If there is a strongly useful natural property in the sense of
  \citet{DBLP:journals/jcss/RazborovR97} --- an efficient algorithm
  that can distinguish between random strings and strings of
  non-trivial circuit complexity --- then general polynomial-sized
  Boolean circuits are learnable over any efficiently samplable
  distribution in randomized polynomial time, given membership queries
  to the target funuction.
\end{itemize}

 \end{abstract}


\section{Introduction}
\label{sec:introduction}

Simple concepts should be efficiently learnable. Exploring this
intuition via formal measures of complexity --- such as VC dimension,
Littlestone dimension, and sample compression cost --- drives progress
in learning theory \citep{DBLP:journals/jacm/BlumerEHW89,
DBLP:journals/ml/Littlestone87, DBLP:journals/jacm/MoranY16}. Recent
work has taken an \emph{algorithmic} perspective: some learning
problems can be reduced to estimating the complexity of a given
concept \citep{ DBLP:conf/coco/CarmosinoIKK16,DBLP:conf/coco/OliveiraS17}. This is different
from characterizing the sample complexity of learning via a combinatorial dimension, 
because the relationship is inherently algorithmic.

However, some learning algorithms 
are distribution-specific: 
they only produce accurate hypotheses for a target function $f$ under the
\emph{uniform distribution} over the domain.  
This is quite limiting, as most
distributions of interest in nature are \emph{not} uniform. For
example, the accuracy of image classification is certainly not
assessed under the uniform distribution on pixels. Formalizing this
observation, the Probably Approximately Correct (PAC) setting of
\citet{DBLP:journals/cacm/Valiant84} is distribution-free: it requires
that a learning algorithm produce an accurate hypothesis over
\emph{any} given distribution.

In this paper, we consider the intermediate setting of learning over
efficiently samplable distributions. We reduce the task of learning a concept class $\Lambda$ over a
samplable distribution $\mu$ to learning $\Lambda$ over the uniform distribution
and inverting a sampler for $\mu$. As a concrete application of this
reduction, we give the first membership-query learning algorithm over
product distributions for $\AC^0[p]$ --- constant-depth circuits with AND, OR, NOT, and
counting modulo $p$ gates for any fixed prime number $p$.
This generalizes previous work that could only learn $\AC^0$
circuits --- which lack counting modulo $p$ gates --- over product
distributions \citep{DBLP:conf/colt/FurstJS91}. Though learning
$\AC^0$ can be done with random examples, it is a strictly less
expressive concept class than $\AC^{0}[p]$ \citep{Razborov:aa,
  DBLP:conf/stoc/Smolensky87}.

\begin{theorem}
\label{thm:learn-AC0p-prod}
$\AC^{0}[p]$ is membership-query learnable over
  product distributions in randomized quasipolynomial time.
\end{theorem}

In the general case, we strengthen connections between learning and
complexity lower bounds. Algorithms that distinguish between random
and circuit-compressible strings are called \emph{natural properties},
because they are implicit in ``natural'' proofs of complexity lower
bounds for Boolean circuits \citep{DBLP:journals/jcss/RazborovR97}. If
there is a natural property against general Boolean circuits, we can
use it twice: first to learn circuits over the uniform distribution,
and second to invert any efficiently samplable distribution over
examples. This construction expands applicability of the ``learning
from natural properties'' technique from only the uniform distribution
to every polynomial-time samplable distribution (\textsc{Psamp}) ---
arguably the most natural class of distributions over which to measure
the accuracy of any hypothesis.

\begin{theorem}
  \label{thm:natP-to-learnPSAMP}
  If there is a strongly useful $\P$-natural property, then general
  polynomial-size circuits can be learned over any distribution in
  \textsc{PSamp} using membership queries in randomized polynomial
  time.
\end{theorem}

\subsection{Our Approach}
\label{sec:our-approach}

Our main technical lemma reduces learning over a distribution $\mu$ to
learning over the uniform distribution and inverting a sampler for
$\mu$. To sketch the construction, let $f$ be a target function,
$S_\mu$ be the efficient machine that transforms uniformly random bits
into samples from $\mu$, and let $I_\mu$ be an inverter for
$S_\mu$. Consider the composed function $g(r) = f(S_\mu(r))$ --- given
coins for the sampler as input, $g$ evaluates $f$ over the resulting
sample from $\mu$. By assumption, we can learn $g$ over the uniform
distribution, getting a preliminary hypothesis circuit $h$. Our final
hypothesis is $h'(x) = h(I_\mu(x))$. Intuitively, since $I_\mu(x)$
maps $x$ back to sampler coins $r$ such that $S_\mu(r) = x$ and $h$ is
a ``good'' hypothesis for $g$, the composed function $h'$ will
generally agree with $f$ on inputs sampled from $\mu$. This is a
natural generalization of the ``indirect approach'' used by
\citet{DBLP:conf/colt/FurstJS91} to learn $\AC^{0}$ over product
distributions; see Section \ref{sec:related-work} for a detailed comparison.

We generalize and extend the ``indirect approach'' by selecting an appropriate definition of
``successful inversion'' such that the simple reduction above is correct. We use
\emph{distributional inversion}, which requires an inverter for the
function $g$ to produce, on input $y$, \emph{uniformly random}
pre-images of $y$ from $g^{-1}(y)$. This ``nice'' guarantee on the
distribution over pre-images allows us to reduce to learning $g$
over the uniform distribution.

Distributional inversion was introduced to show that one-way functions
--- easy to compute but hard to invert --- are essential for
cryptography \citep{DBLP:conf/focs/ImpagliazzoL89}. Instead, we use
the definition and associated cryptographic techniques to give
(conditional and unconditional) \emph{upper bounds} in learning
theory.  Then, since from \citet{DBLP:conf/coco/CarmosinoIKK16} we
already know (1) natural properties imply membership-query learning
over the uniform distributon for general circuits and (2)
membership-query learning of $\AC^0[p]$ over the uniform distribution,
we work to produce expressive distributional inverters
for natural classes of distributions each setting.

\paragraph{(1) Learning over $\psamp$ From Natural Properties.}  We
compose existing results showing that natural properties can break any
standard one-way function \citep{DBLP:conf/coco/ImpagliazzoKV18} with
a uniform cryptographic construction of distributional one-way
functions from standard one-way functions
\citep{DBLP:conf/focs/ImpagliazzoL89}. Because the proofs of these
results are implicitly constructive --- they give algorithms that
reduce the task of finding uniformly distributed pre-images to finding
\emph{any} pre-image --- we can use them to reduce distributional
inversion for any samplable distribution to natural properties.

\paragraph{(2) Learning $\AC^0[p]$ Over Product
Distributions.}  
The above construction requires natural properties
able to distinguish random strings from strings compressible by
linear-depth circuits. But the natural properties that we actually
know just distinguish $\AC^0[p]$-compressible from random strings ---
and $\AC^{0}[p]$ only contains \emph{constant-depth} circuits!  Even
against $\TC^0$ circuits, let alone log-depth, natural properties are
unknown and indeed unlikely to exist
\citep{DBLP:journals/jacm/NaorR04}. So, instead, we identify a
well-studied class of distributions --- product distributions
--- for which distributional inverters and samplers exist
unconditionally. Crucially, the samplers are computable in $\AC^0$ ---
this ensures that the natural properties for $\AC^0[p]$ suffice to
learn over the uniform distribution.

\subsection{Related Work}
\label{sec:related-work}

\paragraph{Utility \& Hardness of Compressability Testing} To
formalize ``compressability testing,'' we refer to the Minimum Circuit
Size Problem ($\MCSP$): given the truth table of a Boolean function
$f$ and a number $s$, does $f$ have a circuit of size at most $s$? It
is a major open question to understand the hardness of $\MCSP$ --- the
problem is clearly in $\NP$, but not known to be in $\P$, and
determining if $\MCSP$ is $\NP$-hard is a major open question. The
``simple'' gadget reductions that suffice for textbook $\NP$-hardness
results cannot be used to show $\NP$-hardness of $\MCSP$
\citep{DBLP:journals/toc/MurrayW17}. However, there are efficient
randomized Turing reductions from factoring, graph isomorphism, and
any problem with a Statistical Zero-Knowledge proof system to $\MCSP$
\citep{DBLP:journals/siamcomp/AllenderBKMR06,
DBLP:journals/iandc/AllenderD17}. All these reductions exploit $\MCSP$
to invert efficiently-computable functions that encode information
about the problem instance. We use this ability (as Lemma
\ref{lem:weakI-to-NatProof}) in our reduction from inverting samplable
distributions to natural properties.

\paragraph{Learning from $\MCSP$} Different variants of $\MCSP$ are
useful for different learning tasks. With membership queries over the
uniform distribution, there is in fact an \emph{equivalence} between
efficient compressability testing and efficient learning
\citep{DBLP:conf/coco/OliveiraS17}. If we have a \emph{tolerant}
natural property --- which can detect strings that are just
\emph{close} to circuit-compressible strings --- then we can learn
over the uniform distribution even when some fraction of membership
queries are adversarially corrupted
\citep{carmosino_et_al:LIPIcs:2017:7584}. $\MCSP$ for \emph{partial}
truth tables --- which allow some ``don't care'' entries --- can even
be used to learn from random examples
\citep{DBLP:conf/coco/IlangoLO20}. Similar to the works above, these
learning results all use (variants of) $\MCSP$ to break a pseudorandom
object that encodes information about the target concept: they are
distinguisher to predictor transformations. We combine both types of
reductions to expand the class of learning problems that reduce to
standard natural properties.

\paragraph{Learning vs. One-Way Functions} The existence of one-way
functions is related to (in)feasibility of learning in a variety of
settings. \citet{DBLP:conf/stoc/KearnsMRRSS94} introduced the setting
of learning a \emph{distribution} --- building an approximate sampler
from random examples, instead of an approximate function. If one-way
functions exist, then there are distributions that are easy to sample
from but hard to learn in this setting. \emph{Inductive inference} is
the problem of observing an efficiently generated sequence and then
predicting the next element
\citep{DBLP:journals/iandc/Solomonoff64a}. Inductive inference (in an
average-case setting) is possible if and only if one-way functions do
not exist \citep{DBLP:conf/focs/ImpagliazzoL90}. Later work even
showed that inductive inference for adaptively changing sequences is
possible if and only if one-way functions do not exist
\citep{10.1145/1143844.1143926}. Though we use some of the same
relationships between cryptographic primitives, our results are for a
different, PAC-type setting.

\paragraph{Learning Constant-Depth Circuits} The techniques used in
proving complexity lower bounds can often be re-purposed to build
learning algorithms. \citet{DBLP:conf/focs/LinialMN89} used the lower
bounds against $\AC^0$ of \citet{DBLP:journals/mst/FurstSS84} to learn
$\AC^0$ over the uniform distribution. Then,
\cite{DBLP:conf/colt/FurstJS91} generalized this approach to learn
$\AC^0$ over product
distributions. \citet{DBLP:conf/coco/CarmosinoIKK16} used the lower
bounds against $\AC^0[p]$ of \citet{Razborov:aa,
  DBLP:conf/stoc/Smolensky87} to learn $\AC^0[p]$ over the uniform
distribution (with membership queries). It was therefore natural to
ask if we can obtain learning for $\AC^{0}[p]$ over product
distributions via similar ideas.

\citet{DBLP:journals/mst/FurstSS84} give two different techniques for
learning $\AC^{0}$ over product distributions. The first generalizes
the underlying Fourier techniques of \citet{DBLP:conf/focs/LinialMN89}
to account for evaluation on product distributions. This technique is
not applicable to $\AC^{0}[p]$, which lacks the required Fourier
concentration properties. We generalize \emph{the analysis of} their
second approach, which composes an $\AC^{0}$-sampler for product
distributions with the target function, and an inverter with the
hypothesis. Essentially, we use a stronger definition of ``successful
inversion.'' They require only that the inverter produce a
distribution that is close --- as a vector of reals --- to the uniform
distribution. This suffices to extend Fourier-based learning to
product distributions, but could fail for algorithms that use
membership queries. \citet{DBLP:journals/mst/FurstSS84} noted that
``it is not at all clear'' how their second technique could be used on
an arbitrary distribution. Our analysis using distributional inversion
(Lemma \ref{lem:flex-learn-reduction}) gives straightforward
conditions under which the technique works for \emph{any samplable}
distribution.

 \section{Preliminaries}
\label{sec:preliminaries}
We denote by $\circ$ concatenation of strings, and by $\cU_{m}$ the
uniform distribution over $m$ bits.  Let $\mathcal{D}_p$ be a
distribution over $\bool$ with $p$ the probability of a 1. A product
distribution $\mathcal{D} = \mathcal{D}_{p_1} \times \mathcal{D}_{p_2}
\times \dots \times \mathcal{D}_{p_n}$  is a distribution over
$\bool^n$ where each coordinate $x_i$ is independently 1 with
probability $p_i$. We say that $\mathcal{D}_p$ is \emph{concise} (has
a concise description) if there exists a constant $k$ such that $p =
s/2^k$, where $s$ is a $k$-bit binary string encoding an integer; we
refer to $k$ as the \emph{precision} for $p$.

We will distinguish between uniform and non-uniform complexity
classes. A complexity class $\Lambda$ is generally defined by
restricting Boolean devices by size, depth, or admissible gates. We
furthermore say that a function $f\colon \bool^\star \to \bool^\star$ is
in $\P$-uniform $\Lambda$ if there is a polynomial-time algorithm
that, on input $1^n$, prints a $\Lambda$-device that computes
$f_n \colon \bool^n \to \bool^\star$, the ``slice'' of $f$ on $n$-bit
inputs.

\subsection{Sampling \& Inversion}
\label{sec:samplers-inversion}

\begin{definition}[$t(n)$-Samplable Ensemble]
  An ensemble $\mu = \{ \mu_{n} \}_{n \in \mathbb{N}}$ is a sequence
  of probability distributions over binary strings of length $n$. The
  ensemble is $t(n)$-samplable if there is a machine $S_{\mu}$ running
  in $t(n)$ time such that, for every $n$ and for every
  $x \in \bool^n$:
  \[
    \Pr\limits_{r \sim \mathcal{U}}[S_{\mu}(1^{n},r) = x] = \mu_{n}(x)
  \]
\end{definition}

The class of ``efficiently samplable'' distributions is
$\textsc{Psamp}$, every $\poly$-samplable ensemble. We will invert
samplers for $\mu$ to learn over $\mu$. That is, we will learn over
distributions with samplers that are \emph{not} one-way functions
(OWFs) --- easy to compute but hard to invert. OWFs are fundamental
objects required for cryptography \citep{DBLP:conf/focs/Yao82a,
  goldreich2001}. Theorems about OWFs are often proved via efficient
reductions between inversion problems; these reductions are useful
sub-routines in our learning setting. Below we define two fundamental
inversion problems by negating standard definitions of OWFs.

\begin{definition}[Invertible Functions]
  Let $f \colon \bool^{\star} \to \bool^{\star}$ be a function on
  bitstrings. We say that $f$ is $(\eta, A)$-invertible if there
  exists a probabilistic polynomial-time oracle Turing Machine $I$
  running in time $t(n)$ such that:
  \[
    \Pr_{x \sim \mathcal{U}_{n}}\left[ f(I^A(f(x))) = f(x) \right]
    \geq \eta
  \]
\end{definition}

Whenever a probability is taken over a probabilistic algorithm, the
coins used by the algorithm are implicitly part of the sample space.
When $\eta$ is $1/\poly(n)$, we call the inversion ``weak'' because it
is not particularly likely, but still noticeable by an efficient
machine. When $\eta$ is $(1 - 1/\poly(n))$ we call the inversion
``strong.''

The above definition ignores the distribution on $f(x)$ pre-images
produced by the inverter. It only prescribes the time complexity of
inversion and the probability that \emph{some} pre-image will be
found. We will require that the inverter produce ``well-distributed''
(i.e., uniformly random) pre-images for any given input. To define
this, we first give a strong information-theoretic notion of
distributional closeness.

\begin{definition}[Statistical Indistinguishability]
  Probability distributions $\mu^{0}_n$ and $\mu^{1}_n$ on $\{0,1\}^n$
  are \textbf{statistically indistinguishable within $\delta$}, 
  denoted $\mu^{0}_n \equiv_{\delta} \mu^{1}_n$,  if
  \[
    \forall T \subseteq \{0,1\}^n
    \left| \Pr_{x \sim \mu^{0}_{n}}[x \in T]
      - \Pr_{x \sim \mu^{1}_{n}}[x \in T] \right| \leq \delta
  \]
\end{definition}

Then, we negate the definition of a distributionally one-way function
\citep{DBLP:conf/focs/ImpagliazzoL89, RussellThesis} to define
well-distributed inversion. Though these definitions are stated for
general functions, our applications will always invert the sampler for
a probability ensemble.

\begin{definition}[Distributionally Invertible Function]
  \label{def:dist-invert}
  We say that a function $f\colon \bool^{\star} \to \bool^{\star}$ is
  $(\delta, A)$-distributionally invertible if there is an efficient
  probabilistic oracle Turing Machine $I$ such that the distributions
  $x \circ f(x)$ and $I^A(f(x)) \circ f(x)$ where
  $x \sim \mathcal{U}_{n}$ are statistically indistinguishable within
  $\delta$ for all but finitely many lengths $n$. That is,
  \[
    x \circ f(x) \equiv_{\delta} I^A(f(x)) \circ f(x)
    ~~\text{where} ~x \sim \mathcal{U}_{n}.
  \]
\end{definition}

In particular, we'll need distributional inversion for all
polynomially small statistical distances. So we say that a function
$f$ is $(1/\poly, A)$-distributionally invertible if, for every
constant $c > 0$, we have that $f$ is $(1/n^c, A)$-distributionally
invertible.

\subsection{Learning Model}
\label{sec:learning-theory}

Our learning model is a relaxation of Valiant's Probably Approximately
Correct (PAC) setting \cite{DBLP:journals/cacm/Valiant84}.
We allow queries, target a specific distribution, and only require
inverse-polynomial error and failure bounds. Formally, let $f$ be a
Boolean function. Our learners are allowed \emph{membership queries}
to $f$, meaning that it may query an input $x \in \{0,1\}^{n}$ and get
back the value $f(x)$ from an oracle in unit time.

\begin{definition}[Query learning over the ensemble $\mu$]
  \label{def:query-learn}
  Let $\Lambda$ be any class of Boolean functions, and let
  $\mu = \{\mu_{n}\}$ be a probability ensemble. We say that
  \emph{$\Lambda$ is query-learnable over $\mu$} if, for any error
  bound $\varepsilon \in \poly^{-1}(n)$ and failure bound
  $\delta \in \poly^{-1}(n)$, there exists an algorithm $A$ that,
  given membership queries to any $n$-bit $f \in \Lambda$, outputs a
  hypothesis $h$ such that
  $\Pr_{x \sim \mu_{n}}[h(x) \neq f(x)] \leq \varepsilon(n)$, with
  probability at least $1 - \delta(n)$ over its internal
  randomness. The runtime of $A$ is measured as a function
  $T = T(n, ~1/\delta, ~1/\varepsilon)$.
\end{definition}

Naturally, we say that $\Lambda$ is learnable over a \emph{class} of
ensembles $\cD$ --- such as \textsc{Psamp} --- if $\Lambda$ is
learnable over every $\mu \in \cD$. Note that a different algorithm
may be used to learn over each distribution in the class, and indeed
to guarantee each inverse-polynomial error and failure rate. However,
our algorithms are uniform with respect to these parameters. Given the
code of a sampler for $\mu$ and the functions $\varepsilon$ and
$\delta$, we can efficiently print the code of $A$ for learning over
$\mu$ to error $\varepsilon$ with failure probability $\delta$.

\subsection{Natural Properties}
\label{sec:natural-properties}
Let $\cF_n$ be the collection of all Boolean functions on $n$ input
variables. A \emph{combinatorial property} of Boolean functions is a
sequence of subsets of $\cF_n$, one for each $n$. We denote a circuit
class by $\Lambda$, and a machine class by $\Gamma$.

\begin{definition}[Natural Property]
  A combinatorial property $\cR = \{\cR_n\}_{n\in\mathbb{N}}$ is
  $\Gamma$-Natural against $\Lambda$ with density
  $\delta_n : \mathbb{N} \to [0,1]$ if it has the following
  properties:
  \begin{itemize}
  \item \textbf{Constructivity:} When $f \in \cF_n$ is given as a
    truth table --- a $2^n$-bit string --- the predicate
    $f_n \overset{?}{\in} \cR_n$ is computable in $\Gamma$.
  \item \textbf{Largeness:} $|\cR_n| \geq \delta_n \cdot |F_n|$
  \item \textbf{Usefulness:} For any sequence of functions $f_n$, if
    $f_n \in \Lambda$ then $f_n \not\in \cR_n$ for all but finitely
    many $n$.
  \end{itemize}
\end{definition}

Denote by $\SIZE[s(n)]$ the set of functions computed by unrestricted
Boolean circuits with at most $s(n)$ gates on $n$ input bits. We then
call a natural property $\cR$ \emph{strongly useful} if there is some $a$, 
$0 <a <1 $, such that $\cR$ is useful against $\SIZE[2^{an}]$ and has
density $\delta_n \geq 1/2$.
 \section{Learning Over $\mu$ Reduces to Learning Over $\cU$ and
  Inverting $\mu$}
\label{sec:flexible-learn}

For efficiently samplable distributions $\mu$, learning a concept
class $\Lambda$ over $\mu$ reduces to learning $\Lambda$ over the
uniform distribution and producing a distributional inverter for
$\mu$. To sketch the argument, let $\Lambda$ be any class of functions
that is closed under composition and contains the sampler for
$\mu$. Then, the $\mu$-sampler composed with the target concept $f$
will also be a function $f_{\mu} \in \Lambda$ --- and thus learnable
over the uniform distribution. If we compose the output hypothesis $C$
that results from learning $f_\mu$ with a distributional inverter
$I_\mu$ for $\mu$, inputs distributed according to $\mu$ will be
indistinguishable from the coins used to sample from $\mu$.

Then, we show that $C(I_\mu(\cdot))$ approximates $f$ over the
distribution $\mu$, because $C$ approximates $f_{\mu}$ over the
uniform distribution and the inverter is ``good.'' Intuitively, if the
event ``$C(I_\mu(\cdot))$ is wrong about $f$'' occurs ``too much''
compared to the event ``$C$ is wrong about $f_{\mu}$'', we can use
this gap as a statistical distinguisher for the ability of $I_{\mu}$
to invert --- a contradiction.

\begin{lemma}[Inversion to Shift Benchmark Distribution]
  \label{lem:flex-learn-reduction}
  Let $\mu$ be any ensemble with a $(1/\poly, A)$-distributionally
  invertible sampler $S_{\mu} \in$ $\P$-uniform $\Lambda$. If
  $\Lambda$ is closed under composition and query-learnable over the
  uniform distribution, then $\Lambda$ is query-learnable over $\mu$.
\end{lemma}

\begin{proof}
  Fix arbitrary $\alpha, \beta \in \poly^{-1}(n)$ to given error and
  failure bounds, as in Definition~\ref{def:query-learn}. Let $f$ be
  any $n$-bit function in $\Lambda$, and $r_{S}(n)$ be the number of
  random bits required by sampler $S_{\mu}$ to simulate
  $\mu_{n}$. Define the composed function
  $f_{\mu}\colon \bool^{r_{S}(n)} \rightarrow \bool$ as
  $f_{\mu}(w) = f(S_{\mu}(1^{n}, w))$. Because $\Lambda$ is closed
  under composition, $f_{\mu} \in \Lambda$.

  Furthermore, given membership query access to $f$ and using the
  uniformity of $S_{\mu}$, we can simulate membership queries to
  $f_{\mu}$. Therefore, since $\Lambda$ is query learnable over
  $\mathcal{U}$, we can efficiently approximate $f_{\mu}$ over the
  uniform distribution to within any inverse-polynomial error bound
  and failure rate. In particular, since
  $\alpha, \beta \in \poly^{-1}(n)$, there is a learner $L$ that
  prints (with probability at least $1 - \beta$) a circuit $C$ such
  that $\Pr_{w \sim \cU}[C(w) \neq f_{\mu}(w)] \leq \alpha/2$. Our
  algorithm first runs $L$ with simulated membership queries to
  produce such a $C$. The second part of our hypothesis is an
  appropriate inverter. As guaranteed by the
  $(1/\poly,A)$-invertability assumption and $\alpha \in 1/\poly(n)$,
  let $I_{\mu}$ be an $(\alpha/2, A)$-distributional inverter for
  $S_{\mu}$. Denote by $r_{I}(n)$ the number of random bits required
  by $I_{\mu}$ to invert $S_{\mu}$. Output $C'(x) = C(I_{\mu}^A(x))$
  as our hypothesis for $f$ over $\mu_{n}$. Notice that the output
  hypothesis will require an oracle for $A$ and random bits to run the
  inverter.

  We bound the error of $C'$ over $\mu$ by combining the error of $C$
  over $\mathcal{U}$ and the statistical indistinghisiability of
  $I_{\mu}$. Define a statistical test $\err(w, x)$ which is equal to 1
  if $C(w) \neq f(x)$ and 0 if $C(w) = f(x)$. Consider the following
  two distributions:
  \begin{align*}
    \mathcal{S}
    &= w \circ S_\mu(1^{n}, w) ~~\text{where } w \sim \cU_{r_{S}(n)} \\
    \mathcal{I}
    &= I_\mu(S_\mu(1^{n},w),z) \circ S_\mu(1^{n}, w)
      ~~\text{where } w \sim \cU_{r_{S}(n)} \text{ and } z \sim \cU_{r_{I}(n)}
  \end{align*}
  Because $I_{\mu}$ distributionally inverts $S_{\mu}$, we know
  $\mathcal{S} \equiv_{\alpha/2} \cI$. That is, suppressing the
  parameter $1^{n}$ for readability, we have the following bound for
  \emph{any} statistical test, and so for $\err$ in particular:
  \[
    \left|
      \Pr_{w,z \sim \cU}
      [\err(I_\mu(S_\mu(w),z),~ S_{\mu}(w))]
      - \Pr_{w \sim \cU}[\err(w, S_{\mu}(w))]
    \right| \leq \alpha/2.
  \]
  Unrolling the definition of $\err$, this becomes:
  \[
    \left|      
      \Pr_{w,z \sim \cU}
      [C(I_\mu(S_\mu(w),z)) \neq f(S_{\mu}(w))]
      - \Pr_{w \sim \cU}[C(w) \neq f(S_{\mu}(w))]
    \right| \leq \alpha/2.
  \]
  By definition of $f_{\mu}$, and because $S_{\mu}$ samples from
  $\mu$, we  obtain:
  \[
    \left|
      \Pr_{z \sim \cU, x \sim \mu_{n}}
      [C(I_\mu(x,z)) \neq f(x)]
      - \Pr_{w \sim \cU}[C(w) \neq f_{\mu}(w)]
    \right| \leq \alpha/2.
  \]
  Finally, we are left with a bound on the difference in error rates
  between $C'$ and $C$ over their respective input distributions.
  \[
    \left|
      \Pr_{z \sim \cU, x \sim \mu_{n}}
      [C'(x,z) \neq f(x)]
      - \Pr_{w \sim \cU}[C(w) \neq f_{\mu}(w)]
    \right| \leq \alpha/2
  \]
  Re-arranging and using the error bound for $L$ completes the proof:
  \[
    \Pr_{z \sim \cU, x \sim \mu_{n}}
    [C'(x,z) \neq f(x)]  
    \leq \alpha/2 + \alpha/2 \leq \alpha
  \]

  Thus, $f$ is learnable over $\mu$ with the error $\alpha$ and
  failure probability $\beta$.

\end{proof}

 \section{Natural Properties $\implies$ Query-Learning over $\psamp$}
\label{sec:flex-learn-from-NatPr}

If there is a strongly useful $\P$-Natural property against $\Ppoly$,
then $\Ppoly$ is query-learnable over \textsc{Psamp}. To show this
using the reduction above (Lemma~\ref{lem:flex-learn-reduction}), we
just need to construct distributional inverters for all of
\textsc{Psamp} from natural properties. This will follow by efficient
and uniform composition of hardness results about Natural Properties
with classical relationships between various kinds of one-way
functions. Below, we restate these results as reductions from the
inversion problems defined in Section \ref{sec:samplers-inversion}.

\begin{lemma}[Weak Inversion Reduces to Natural Properties]
  \label{lem:weakI-to-NatProof}
  Let $\cR$ be any strongly useful natural property. There exists
  $k \in \nat$ such that any poynomial-time computable function on
  bitstrings is $(1/n^k, \cR)$-invertible.
\end{lemma}

Lemma~\ref{lem:weakI-to-NatProof} first appeared as part of hardness
results for testing resource-bounded Kolmogorov complexity
\citep{DBLP:journals/siamcomp/AllenderBKMR06}. Later, it was used to
show that if there is indistinguishability obfuscation against
$\BPP^\cR$, then $\NP \subseteq \ZPP^\cR$ (Lemma 45 of
~\cite{DBLP:conf/coco/ImpagliazzoKV18}). Note that the (weak)
inversion success probability is a fixed inverse polynomial $n^{-k}$.

\begin{lemma}[Strong Inversion Reduces to Weak Inversion]
  \label{lem:strongI-to-weakI}
  Let $f$ be an $n$-input Boolean function, and let $q$ and $p$ be any
  polynomials. Define the $t(n)$-\textbf{direct product} of $f$ as
  $f'(x_1,\dots,x_t(n)) = f(x_1)\circ \dots \circ f(x_{t(n)})$. Suppose $f'$
  is $(1/q(m), A)$-invertible for some oracle $A$. Then, $f$ is
  $(1 - 1/p(n), A)$-invertible.
\end{lemma}

Lemma~\ref{lem:strongI-to-weakI} first appeared implicitly in
\cite{DBLP:conf/focs/Yao82a} --- see Theorem 2.3.2 of
\cite{goldreich2001} for a proof. Here, it is important for us that
even a fixed inverse-polynomial weak inversion probability amplified
to strong inversion probability $1 - p(n)$ for \emph{any}
inverse-poynomial $p(n)$ --- because of Lemma
\ref{lem:weakI-to-NatProof}, which only gives fixed inverse-polynomial
probaility of inversion. The cost for this flexibility is just a
larger polynomial runtime of the resulting inverter, proportional to
the increased success probability.

Below, let $H_{n,m}$ be a family of hereditarily universal hash
functions from $n$-bit strings to $m$-bit strings. We write
$z \downharpoonleft_{i}$ to mean ``the first $i$ bits of binary string
$z$.''

\begin{lemma}[Distributional Inversion Reduces to Strong Inversion]
  \label{lem:distI-to-strongI}
  Let $f$ be an $n$-input Boolean function. Define the
  $c$-\textbf{truncating hash} of $f$ as
  $f'(h,i,x) = h \circ i \circ f(x) \circ h(x)\downharpoonleft_{i}$,
  where $i \in [m], ~h \in H_{n,m}$, and $m = n + (6c+6)\log
  n$. Suppose $f'$ is $(1 - 1/n^{6c}, A)$-invertible for some oracle
  $A$. Then, $f$ is $(2/n^{c}, A)$-distributionally invertible.
\end{lemma}

Lemma~\ref{lem:distI-to-strongI} was used to show that many
cryptographic tasks are equivalent to the standard definition of a
one-way function \cite{DBLP:conf/focs/ImpagliazzoL89} --- see Theorem
4.2.2 of \cite{RussellThesis} for a proof. We are now ready to obtain
distributional inverters from natural proofs by composition of the
above.

\begin{lemma}[Natural Proofs Imply Distributional Inverters]
  \label{lem:natP-to-distI}
  If there is a strongly useful $\P$-natural property $\cR$, then
  every $\mu$ in \textsc{Psamp} is $(\poly, \cR)$-distributionally
  invertible.
\end{lemma}

\begin{proof}
  Fix an arbitrary ensemble $\mu \in$ \textsc{Psamp}, and arbitrary
  polynomial $p(n)$ as the target distribution and parameter of
  statistical distance for distributional inversion. Let $f$ be the
  polynomial-time sampler associated with $\mu$. Choose $c$ such that
  $1/p(n) < 2/n^c$. Then, let $h$ be the $c$-truncating hash of $f$
  and $g$ be the $n^{6c}$-direct product of $h$. Because $f$ is
  efficiently computable, so are $h$ and $g$ --- we only took a
  polynomial-sized direct product of $h$ and there are efficient
  universal hash families \citep{DBLP:journals/jcss/CarterW79}.

  We now unwind the above transformation of $f$ to obtain a
  distributional inverter. First, Lemma \ref{lem:weakI-to-NatProof}
  applies and $g$ is $(1/n^k, \cR)$-invertible for some fixed
  $k$. Then, Lemma \ref{lem:strongI-to-weakI} applies to $g$ and $h$,
  and so $h$ is $((1 - 1/n^{6c}), \cR)$-invertible. Finally, Lemma
  \ref{lem:distI-to-strongI} means that $f$ is
  $(1/p(n), \cR)$-distributionally invertible. As $p(n)$ was an
  arbitrary polynomial and $\mu \in$ \textsc{Psamp} was arbitrary, the
  lemma follows.
\end{proof}

Having obtained the required inverters from natural properties, we
formally recall that sufficiently useful natural properties imply
efficient learning of $\Ppoly$ over $\cU$.

\begin{theorem}[Natural Proofs Imply Learning over $\cU$
  \citep{DBLP:conf/coco/CarmosinoIKK16}]
    \label{thm:natP-to-learnU}
  If there is a strongly useful $\P$-Natural property, then $\Ppoly$
  is query-learnable over $\cU$ in randomized polynomial time.
\end{theorem}

Theorem \ref{thm:natP-to-learnPSAMP} now follows by combining the
learning algorithm of Theorem \ref{thm:natP-to-learnU} with the
distributional inverters of Theorem \ref{lem:natP-to-distI} via Lemma
\ref{lem:flex-learn-reduction} for each $\mu \in \psamp$.

 \newcommand{\Samp}{\texttt{Samp}}
\newcommand{\BitInv}{\texttt{BitInv}}
\newcommand{\ProdSamp}{\texttt{ProdSamp}}
\newcommand{\ProdInv}{\texttt{ProdInv}}

\section{$\AC^0[p]$ is Query-Learnable Over Bounded Product Distributions}
\label{sec:ac02-mq-learnable-prod}

We already know that $\AC^0[p]$ is query-learnable over $\cU$ in randomized quasipolynomial time. So, to learn over product distributions via Lemma \ref{lem:flex-learn-reduction}, we just need to construct
samplers for product distributions that are both expressible in $\AC^0[p]$ and $(\poly, \emptyset)$-distributionally invertible. As we need the explicit description of the distribution (that is, the list of biases $p_i$), we require all our distributions to be concise.  

We begin by explicitly sampling and inverting a single bit with bias $p\in (0,1)$, where $p$ is given with precision $k$ in binary.   More precisely, we interpret a $k$-bit binary string encoding $p$ as the numerator of a fraction with denominator $2^k$, or, equivalently,  as a dyadic $k$-bit number with exact $k$-bit expansion $bin(p)=b_1 \dots b_k$, so that 
\(
  p = \sum_{i=1}^{k} b_i \times 2^{-i}
\)

With the dyadic representation, to sample a single bit with probability $p$ it is enough to sample a $k$-bit string $r$ uniformly at random. This $r$ is also interpreted as a dyadic number, and the sampler returns 1 iff $r < p$ and 0 otherwise. Let \texttt{Samp}$(p)$ be the procedure that performs this operation; if we need to specify randomness $r$ explicitly, we will use the notation \texttt{Samp(p;r)}. To sample from a product distribution, sample each bit with its respective bias, using fresh randomness for each sample; let \texttt{ProdSamp}$(p_1, \dots, p_n)$ be the respective sampler. Note that both the $\Samp$ and $\ProdSamp$ are computable by $\AC^0$ circuits with auxiliary inputs for random bits (sampling of individual bits by $\ProdSamp$ can be done in parallel, since they are sampled independently). 

Now, we construct a distributional inverter for \texttt{Samp}$(p)$. Let $\cU_k$ be the uniform distribution over $k$-bit binary strings used as randomness in  \texttt{Samp}$(p)$. If we look at $bin(p)$  as a binary number, strings in $[0,bin(p)-1]$ are preimages of 1, and strings in $[bin(p),2^k-1]$ are preimages of 0, so the inverter just needs to return a uniformly random $k$-bit binary number from the corresponding interval.  As the construction uses rejection sampling, we also include a failure probability $\gamma$, which can be amplified to be arbitrarily small (such as $\gamma=1/2^{2n}$). This inverter \texttt{BitInv}$(p,b,\gamma)$ will return a uniformly random preimage of a bit $b$ sampled from $D_p$ by \texttt{Samp}$(p)$, or fail (return $\bot$) with probability at most $\gamma$.   

The bit inverter works as follows. To sample a preimage of $1$, choose $C\leq k$ such that $2^{C-1} \leq p < 2^C$ and sample a uniformly random $C$-bit string $y$, then expand this string to $k$ bits by adding $k-C$ leading 0s, if necessary; let $r$ be the resulting $k$-bit string. Now $r$ can be interpreted as a $k$-bit dyadic number in the range $[0\dots \frac{2^C-1}{2^k}]$. As $p \geq  2^{C-1}$, with probability $\geq 1/2$ it will be $r < p$. In that case, return $r$. Otherwise, sample another  $C$-bit string $y'$. Repeat this process until either $r <p$ is obtained, or the number of rounds (attempts to sample) exceeds $\lceil \log 1/\gamma \rceil$. Since in each round $r$ is sampled independently, and each round succeeds and returns a string with probability $\geq 1/2$, after $\lceil \log 1/\gamma \rceil$ rounds the failure probability will be $1/2^{\lceil \log 1/\gamma \rceil } \leq \gamma$. Finally, to sample a preimage of $0$, sample $r < 1-p$ using the same procedure as sampling $r <p$ above, then return $p+r$. As the failure probability in each round is always $\leq 1/2$ irrespective of $b,p, k, \gamma$, the same  $\lceil \log 1/\gamma \rceil$  number of rounds suffices to make the probability of failure $\leq \gamma$ for sampling the preimage of either bit. 

Note that this procedure does not have to be sequential: for a fixed $\gamma$  and $k$,  the bit inverter can be implemented as an $\AC^0$ circuit with inputs $b,p$, together with   $k \cdot \lceil \log 1/\gamma \rceil$ bits of randomness as auxiliary inputs, since essentially all it does is comparing $k$-bit numbers.  The circuit will have $k+1$ outputs, where one bit denotes success/failure and the rest are the bits of the first $r_i < p$.

\begin{lemma}\label{lem:bitinv-uniform} 
  \texttt{BitInv} samples uniformly from the preimage of \texttt{Samp}. That is, for any $p$,  \[\forall b \in \bool  \forall r, r' \in \Samp^{-1}(b) \ \ \Pr[\BitInv(b,p,\gamma)=r] = \Pr[\BitInv(b,p,\gamma)=r']\] 
 
\end{lemma} 
\begin{proof} 
Without loss of generality, suppose that $b=1$. Then $r,r' < p$, and thus since $bin(p) < 2^C$ they can be written as $r=0^{k-C}y$ and  $r'=0^{k-C}y'$. Moreover,  $r$ and $r'$ are uniquely determined by $y$ and $y'$.  By construction, the probability that \texttt{BitInv} chooses a specific $C$-bit binary string is $1/2^C$. Thus, $\Pr[\BitInv(b,p,\gamma)=r] = \Pr[\BitInv(b,p,\gamma)=r']=1/2^C$ (note that $\Pr[\BitInv(b,p,\gamma)=\bot] = \Sigma_{z \in \{0^{k-C}z_1 \dots z_C \mid z \geq p\}}  1/2^C = (2^C-bin(p))/2^C$  
\end{proof}

Finally, let $\mathcal{D} = \mathcal{D}_{p_1} \times \mathcal{D}_{p_2} \times \dots \times \mathcal{D}_{p_n}$, where $\mathcal{D}_p$ is the distribution sampled by $\texttt{Samp}(p)$. For each $p_i$, let $k_i$ be its precision (that is, number of bits in its dyadic representation).   A preimage of $x=x_1 \dots x_n$ sampled from $\mathcal{D}$ is a list of $n$ strings $r_1, \dots,  r_n$,  where each $r_i$ is uniformly distributed among preimages of $x_i$ (that is, $k_i$-bit strings that when used as randomness in \texttt{Samp}$(p_i)$ result in output $x_i$).   The distributional inverter for $\mathcal{D}$,  \texttt{ProdInv}$(x,p_1 \dots p_n, \gamma)$, obtains each $r_i$ by calling $\texttt{BitInv}(x_i,p_i, \gamma)$; if at least one call to $\texttt{BitInv}$ fails, then \texttt{ProdInv} fails as well.

\begin{lemma}
  \texttt{ProdInv} samples uniformly from preimage of
  \texttt{ProdSamp}.   
    That is, for $x \in \{0,1\}^n$: 
\begin{align*} 
\Pr&[\ProdInv(x)  = (r_1, \dots, r_n)]   = \Pr[\ProdInv(x) = (r'_1, \dots, r'_n)]\\  
    &\forall  (r_1,\dots,r_n),  (r'_1, \dots, r'_n)  \in \ProdSamp^{-1}(x)
\end{align*}

\end{lemma}

\begin{proof} 
Let  $(r_1,\dots,r_n),  (r'_1, \dots, r'_n)  \in \ProdSamp^{-1}(x)$. As a preimage of each $x_i$ is sampled using independent fresh random bits, for every  $i\neq j$, $z_i\in \Samp^{-1}(x_i)$, $z_j \in \Samp^{-1}(x_j)$  \(\Pr[\BitInv(x_i)=z_i \wedge  \BitInv(x_j)=z_j] = \Pr[\BitInv(x_i)=z_i]\cdot \Pr[\BitInv(x_j)=z_j]\  \).  

For each $i$, the distribution over preimages $r_i$ of $x_i$ sampled by $\BitInv$ is uniform by lemma \ref{lem:bitinv-uniform}. Now, a direct product of uniform distributions is equivalent to a uniform distribution over respective direct products, completing the proof.   \end{proof} 

\begin{lemma}\label{lem:prodinv-failure}
  The failure probability of \texttt{ProdInv}$(x,p_1 \dots p_n, \gamma)$ is at most $\gamma n $, for $\gamma < 1/n$.  
\end{lemma}
\begin{proof}
As all calls to \texttt{BitInv} have to be successful for \texttt{ProdInv} not to fail, the probability of success of \texttt{ProdInv} is $(1-\gamma)^n$. Using Taylor series approximation and the $\gamma < 1/n$ assumption, $(1-\gamma)^n \geq 1-\gamma n$.  Thus, the probability of failure of   \texttt{ProdInv} is at most $1 - (1-\gamma)^n \leq 1 - (1-\gamma n) = \gamma n$. 
\end{proof}

With this, we obtain a distributional sampler and inverter for product
distributions computable by a family of $\AC^0$  (provided all biases
$p$  are representable using constant precision). Setting  $\gamma =
2^{-n}$,  $\ProdInv$ succeeds with probability exponentially close to
$1$, and uses $n^2 \cdot \max_i{k_i} = O(n^2)$ random bits to sample
each $x \in\mathcal{D}$. Thus, we can combine it with the
\cite{carmosino_et_al:LIPIcs:2017:7584}  learning algorithm for
$\AC^0[p]$ over the uniform distribution to learn  $\AC^0[p]$ over product distributions.   

\begin{lemma}[Distributional Inverters for concise product distributions]
  \label{lem:prod-inverter}
Concise product distributions have distributional inverters. Moreover, these inverters are computable in $\AC^0$ with auxiliary random bits. 
\end{lemma}

\subsection{Learning $\AC^0[p]$ with distributional inverters }

For any prime $p$, the class $\AC^0[p]$ is closed under composition with $\AC^0$ functions. Since the paired inverter is distributional, this satisfies the requirements for applying our reduction from samplable learning to uniform learning. We now formally recall that $\AC^0[p]$ is indeed learnable over the uniform distribution.

\begin{theorem}[\citep{DBLP:conf/coco/CarmosinoIKK16}]\label{thm:learn-AC0p-unif}
For every prime $p$, the class $\AC^0[p]$ is query-learnable over $\cU$ in randomized quasi-polynomial time.
\end{theorem}

Theorem \ref{thm:learn-AC0p-prod} now follows by combining the learning algorithm of Theorem \ref{thm:learn-AC0p-unif} with the distributional inverters of Lemma \ref{lem:prod-inverter} via Lemma \ref{lem:flex-learn-reduction} for each $\mu$ a concise product distribution.

 \section{Conclusions \& Future Directions}
\label{sec:future-open}

A key difference between the learning algorithms we obtain
from natural properties and the general PAC setting is the quantifier
ordering. Our learning algorithms depend on the distribution; PAC
learning requires that a single algorithm work for all possible
distributions. If a \emph{universal} distributional inverter can be
constructed from a natural property, then this would further close the
gap between the PAC setting and the learning problems that reduce to
natural properties. This is a natural direction for future investigations,
especially given recent progress in constructing one-way functions from
hardness assumptions about compression problems related to the Minimum Circuit
Size Problem \citep{DBLP:conf/stoc/LiuP21}.

We have expanded the scope of the ``indirect'' method
for reducing from learnability over a samplable distribution to 
over uniform. Product distributions are just the most natural and
well-studied example of a target distribution for which the technique 
works unconditionally. What other interesting ensembles of probability distributions
have efficient samplers and distributional inverters? 
\bibliography{FlexNW}

\end{document}